\documentclass{article}

\usepackage[utf8]{inputenc} 
\usepackage[T1]{fontenc}    
\usepackage{hyperref}       
\usepackage{url}            
\usepackage{booktabs}       
\usepackage{amsfonts}       
\usepackage{nicefrac}       
\usepackage{microtype}      
\usepackage{xcolor}         
\usepackage{amsmath}
\usepackage{amsthm}
\usepackage{amssymb}
\usepackage{physics}

\newtheorem{theorem}{Theorem}[section]
\newtheorem{lemma}{Lemma}

\newtheorem{definition}{Definition}
\newtheorem{corollary}[theorem]{Corollary}

\theoremstyle{remark}

\newtheorem{remark}{Remark}

\newcommand{\R}{\mathbb{R}}
\newcommand{\N}{\mathbb{N}}
\newcommand{\X}{\mathcal{X}}
\newcommand{\Y}{\mathcal{Y}}
\newcommand{\Z}{\mathcal{Z}}

\newcommand{\Sp}[1]{\left(#1\right)}
\newcommand{\Mp}[1]{\left[#1\right]}
\newcommand{\Bp}[1]{\left\{#1\right\}}

\newcommand{\LN}{\mathrm{LN}}
\newcommand{\Attn}{\mathrm{Attn}}
\newcommand{\TF}{\mathrm{TF}}
\newcommand{\Ret}{\mathrm{Ret}}
\usepackage{cleveref}
\usepackage{pifont}
\usepackage{tabularx}
\usepackage{wrapfig}
\usepackage{nicefrac}
\usepackage{makecell}
\usepackage{array}
\usepackage{thm-restate}
\usepackage{algorithm}
\usepackage{algorithmic}
\usepackage{fullpage}

\RequirePackage{natbib}

\definecolor{forestgreen}{RGB}{0,155,85}

\newcommand{\red}[1]{{\color{red}#1}}
\newcommand{\green}[1]{{\color{forestgreen}#1}}

\title{On Surjectivity of Neural Networks:\\ Can you elicit any behavior from your model?}

\author{
    Haozhe Jiang\\
    \texttt{ericjiang@berkeley.edu}\\
    \and
    Nika Haghtalab\\
    \texttt{nika@berkeley.edu}\\
}
\date{}

\begin{document}

\maketitle

\begin{abstract}
Given a trained neural network, can any specified output be generated by some input? Equivalently, does the network correspond to a function that is surjective? 
In generative models, surjectivity implies that any output, including harmful or undesirable content, can in principle be generated by the networks, raising concerns about model safety and jailbreak vulnerabilities.
In this paper, we prove that many fundamental building blocks of modern neural architectures, such as networks with pre-layer normalization and linear-attention modules, are almost always surjective. As corollaries, widely used generative frameworks, including GPT-style transformers and diffusion models with deterministic ODE solvers, admit inverse mappings for arbitrary outputs. By studying surjectivity of these modern and commonly used neural architectures, we contribute a formalism that sheds light on their unavoidable vulnerability to a broad class of adversarial attacks.

\end{abstract}

\section{Introduction}
Deep generative models have achieved remarkable success in recent years—spanning natural language processing~\citep{openai2024gpt4technicalreport,touvron2023llama2openfoundation,chowdhery2022palmscalinglanguagemodeling}, computer vision~\citep{imagenteamgoogle2024imagen3,grattafiori2024llama3herdmodels}, and robotics~\citep{kim2024openvlaopensourcevisionlanguageactionmodel,geminiroboticsteam2025geminiroboticsbringingai}. Yet this progress has raised growing safety concerns. Powerful models can be manipulated to produce undesirable or even dangerous content~\citep{zou2023universaltransferableadversarialattacks,wan2023poisoninglanguagemodelsinstruction,ma2024jailbreakingpromptattackcontrollable}, and the risk only intensifies as their capabilities expand. To mitigate these threats, considerable effort has been devoted to data curation and safety fine-tuning~\cite{grattafiori2024llama3herdmodels,openai2024gpt4technicalreport} with the aim of constraining model behavior during training. But a fundamental question remains unanswered: can we ever guarantee that a trained model will not generate harmful content? Or could it be that, given a trained generative model and an arbitrary target output, there always exists an input that produces that output? In mathematical terms, 
 \emph{viewing a generative model as a function from its input space to its output space, we ask whether that function is surjective?}

Formalizing the study of surjectivity in trained models presents significant challenges and deviates from the standard practice of the community. On the one hand, it is clearly too much to expect that \emph{every possible} choice of the parameters yields a model that is surjective. Pathological cases --- such as setting all weights to zero --- can lead to degenerate models that implement constant functions.
But merely observing that some parameter settings lead to surjectivity is equally uninformative, since there is no guarantee that the training process will uncover them.
To address this, we adopt a probabilistic perspective: rather than asking whether \emph{all} or \emph{some} parameter settings yield surjectivity, we ask whether surjectivity holds \emph{almost always}.
That is, for a fixed model architecture, do all except for a measure-zero subset of parameter configurations, lead to functions that are surjective?
This formalism better reflects the practical realities of the training process:
the evolution of parameters depends intricately on the choice of optimizer, data distribution, loss function, and even future training paradigms. These elements introduce randomness into the training process, making the final trained model effectively a draw from a high-dimensional distribution. If surjectivity holds almost everywhere, this implies that regardless of the fine-grained details of the training process, it is exceedingly unlikely that the resulting model would not be surjective.

\paragraph{Technical Results and Toolkit.}
Given the significant departure from existing paradigms of studying neural networks, our perspective also calls for a new toolkit to analyze the input-output behavior of trained models.
 We show that differential topology is the right tool for the job!
Differential topology analyzes smooth manifolds under smooth transformations and offers a lens on the global structure of neural network outputs. This connection is far from accidental---modern networks are constructed from smooth components that make them amenable to optimization via backpropagation. As a result, these models are naturally suited to analysis using tools from differential topology.
We provide a gentle introduction to some of these tools and show how they can be used to study surjectivity of neural networks with relative ease.
\begin{wraptable}{r}{-.5\linewidth}
    \centering
    \begin{tabular}{cc}
        \toprule
        Architecture&Surjectivity\\
        \midrule
        MLP with ReLU/GeLU&\red{\ding{56}}\\
        MLP with Pre-LayerNorm&\green{\ding{52}}\\
        Attention&\red{\ding{56}}\\
        Attention with Pre-LayerNorm&\green{\ding{52}}\\
        Linear Attention (Retention)&\green{\ding{52}}\\
        \bottomrule
    \end{tabular}
    \caption{Partial Summary of Results in \Cref{section:theory}. Wrapping a function $f$ in Pre-LayerNorm is defined with residual connection $f(\LN(x))+x$.
    }
    \vspace{-8pt}
    \label{table:result}
\end{wraptable}

Using these tools, we analyze the surjectivity of core building blocks in modern architectures, including LayerNorm with residual connections and attention mechanisms. In particular, we show that any continuous function wrapped with Pre-LayerNorm is surjective (Theorems~\ref{thm:preln} and \ref{thm:prelnseq})---implying that both Attention and MLP layers are surjective when wrapped in Pre-LayerNorm.
We also establish that linear attention, specifically Retention (\Cref{thm:ret}), is almost always surjective.
Two notable negative results complement these guarantees: softmax Attention itself is not almost always surjective (\Cref{thm:notsur}), and two-layer MLPs with ReLU or GeLU activation are not almost always surjective (\Cref{thm:relu-mlp-not-surj,thm:gelu-mlp-not-surj}).

In \Cref{section:safety}, we discuss the practical implications of surjectivity of these building blocks, using concrete examples in language, vision, and robotics. 
In particular, we show that Transformers, diffusion models with deterministic ODE solvers, and certain policy networks commonly used in robotics are all surjective.
Our work highlights significant obstacles to achieving provably safe architectures in these applications.

{\color{black}
\paragraph{Broader Implications of Surjectivity on Safety.}
In Section~\ref{sec:broad}, we discuss the broader implications of surjectivity on model safety and safety training.
Existing work on jailbreaks has made important progress by identifying and mitigating specific vulnerabilities. However, without a deeper understanding of whether such vulnerabilities are avoidable in principle, research on jailbreaks runs the risk of becoming a cat-and-mouse game of patching symptoms rather than addressing root causes. Our work complements these efforts by offering a more foundational perspective on jailbreaks that highlights a fundamental challenge in creating jailbreak-proof safe AI models, using surjectivity as the formalism.

From a theoretical perspective, surjectivity implies that a model is vulnerable to jailbreaks in principle. That is, every outcome, including those considered harmful by model providers, can be generated by some input. The study of surjectivity also neatly decouples risks that are rooted in an attacker's ability to elicit particular behavior --- which is the main consideration of jailbreaks --- from the
domain-specific risks that arise from having highly capable AI models in certain areas (such as bioweapons, etc.) in the first place.
Given that our results hold under no particular assumptions on the training process (other than acknowledging that elements in the optimization pipeline introduce randomness in the training process), this shows that, at least in theory, safety training on several commonly used model architectures cannot prevent the model from outputting harmful behavior.

Still, surjectivity is an existential property that does not guarantee that inputs for eliciting harmful behavior can be found with efficient computation or a feasible amount of information. We discuss these considerations further in Section~\ref{sec:broad}, highlighting how some existing attacks can be viewed through the lens of surjectivity, what the study of surjectivity adds to the discourse on complementary approaches to safety training for AI safety, and exploring future directions for work that might be of interest to the community.

}

More broadly, the surjectivity of modern architectures prompts a deeper question about the state of research in AI safety: what are the appropriate frameworks for studying safety, jailbreaking, and even copyright risks in generative models?
In all three cases, current evaluations often rely on probing the model's output behavior to examine whether certain harmful behavior or output resembling proprietary information can be elicited through adversarial inputs.
But given that  surjectivity implies that any output can, in principle, be elicited from any model, our work suggests that caution is needed when drawing conclusions about model safety based solely on its output behavior --- especially when inputs can be manipulated outside typical usage patterns.

\subsection{Related works.}

\paragraph{Invertible Architectures.} Invertibility, the property of being both injective and surjective, is a stronger notion than surjectivity which has been studied in neural networks before. \citet{rezende2016variationalinferencenormalizingflows} proposed normalizing flow, which uses invertible functions to model complex distributions. A line of work constructing invertible neural networks thus follows~\citep{dinh2015nicenonlinearindependentcomponents,dinh2017densityestimationusingreal,kingma2017improvingvariationalinferenceinverse,papamakarios2018maskedautoregressiveflowdensity,kingma2018glowgenerativeflowinvertible,durkan2019neuralsplineflows,chen2020residualflowsinvertiblegenerative}. Beyond density estimation, researchers also explored invertible networks motivated by memory savings and representational power~\citep{gomez2017reversibleresidualnetworkbackpropagation,jacobsen2018irevnetdeepinvertiblenetworks,behrmann2019invertibleresidualnetworks,song2019mintnetbuildinginvertibleneural}. 
Specifically, in sequence models invertible architectures have been proposed to save memory~\citep{mackay2018reversiblerecurrentneuralnetworks,kitaev2020reformerefficienttransformer,mangalam2023reversiblevisiontransformers}. 
A key difference to our work is that while prior efforts aimed to modify architectures to ensure invertibility, the modern architectural blocks we study were not designed with invertibility in mind.

\paragraph{Safety.}
Attacks on generative models that lead to safety violations have been extensively studied in prior works.

We start with Language Models. There is a long line of work studying jailbreaks, which means constructing prompt to elicit undesirable behaviors from a trained language model. Jailbreaks can be classified as black-box attacks and white-box attacks. Black-box attacks restrict the attacker's access such that only prompt inputs are allowed, and we do not have knowledge about the model's internal parameters or architecture. We list some methods as follows. Goal-hijacking guide the model to override intention of the original prompt, and follow the attacker's wish by adding additional prompt to the original prompt~\citep{perez2022ignorepreviouspromptattack,liu2024promptinjectionattackllmintegrated}. Another similar method suppresses the model from refusing to answer harmful questions~\citep{wei2023jailbrokendoesllmsafety}. Few-shot jailbreaks manipulate the model by showing it demonstrations of harmful responses~\citep{rao-etal-2024-tricking,wei2024jailbreakguardalignedlanguage,li-etal-2023-multi-step}. Code jailbreaks take advantage of model's power of code comprehension to conceal malicious contents in codes~\citep{kang,liu2024boostingjailbreaktransferabilitylarge}. A line of work tells the language model to role play in a fictional world to let it generate harmful outputs~\citep{liu2024jailbreakingchatgptpromptengineering,deshpande-etal-2023-toxicity,shah2023scalabletransferableblackboxjailbreaks,kang,xu-etal-2024-cognitive}. Some attacks exploit lack of alignment data in low-resource languages to achieve jailbreak~\citep{yong2024lowresourcelanguagesjailbreakgpt4,deng2024multilingualjailbreakchallengeslarge,xu-etal-2024-cognitive}. A similar approach jailbreaks models by communicating in a ciphered texts~\citep{wei2023jailbrokendoesllmsafety,yuan2024gpt4smartsafestealthy,handa2025competencyreasoningopensdoor}. In contrast to black-box approaches, white box approaches allow us to access the whole model (open-source models). The seminal work of Greedy Coordinate Gradient~\citep{zou2023universaltransferableadversarialattacks} use gradient-based method to optimize a suffix in the embedding space to maximize the likelihood of harmful output. Subsequent works following this path and improve the success rate by better optimization strategies\citep{zhang2025boostingjailbreakattackmomentum,hu2025efficientllmjailbreakadaptive,jia2024improvedtechniquesoptimizationbasedjailbreaking,liu2024boostingjailbreaktransferabilitylarge,huang2024semanticguidedpromptorganizationuniversal,geisler2025attackinglargelanguagemodels,huang-etal-2025-stronger,sitawarin2024palproxyguidedblackboxattack,wang2025julijailbreaklargelanguage}. Notably, some works attacks the model through manipulating hidden embedding~\citep{wang2024attngcgenhancingjailbreakingattacks,hu2024drojpromptdrivenattacklarge}. Allowing access to model weights also allow us to finetune the model from safe to unsafe ones~\cite{wan2023poisoninglanguagemodelsinstruction,rando2024universaljailbreakbackdoorspoisoned,liao2024amplegcglearninguniversaltransferable,kumar2024amplegcgplusstronggenerativemodel,paulus2024advprompterfastadaptiveadversarial}. Other safety concerns of generative language models include bias, privacy, misuse, agent safety, and so on. We refer interested readers to \citet{shi2024largelanguagemodelsafety} for a more comprehensive survey.

Vision models has also become very powerful in recent years, and safety concerns rises. When saying vision models we usually refer to vision-language models (VLMs), because most useful vision models nowadays include both modalities. Let us first talk about white-box attacks. One line of works exploits the vision module by constructing adversarial images to let the model output undesirable images or texts~\cite{qi2023visualadversarialexamplesjailbreak,schlarmann2023adversarialrobustnessmultimodalfoundation,bailey2024imagehijacksadversarialimages,madry2019deeplearningmodelsresistant,luo2024imageworth1000lies,shumailov2021spongeexamplesenergylatencyattacks,chen2022nicgslowdownevaluatingefficiencyrobustness}. Another line of works attack the model by exploiting both modalities~\citep{wang2024whiteboxmultimodaljailbreakslarge,li2025imagesachillesheelalignment,luo2024imageworth1000lies,ying2024jailbreakvisionlanguagemodels}. For VLMs, there is a special class of attacks called grey-box attacks. These methods leverage the fact that a lot of vision encoders are CLIP~\citep{radford2021learningtransferablevisualmodels} or BLIP~\citep{li2022blipbootstrappinglanguageimagepretraining} to create better attacks. Like white-box attacks, there are also single-modality~\citep{zhao2023evaluatingadversarialrobustnesslarge,dong2023robustgooglesbardadversarial,niu2024jailbreakingattackmultimodallarge} and cross-modality~\citep{shayegani2023jailbreakpiecescompositionaladversarial} attack methods. Last we introduce black-box methods. One line of works attack models by constructing malicious typography~\citep{gong2025figstepjailbreakinglargevisionlanguage,qraitem2025visionllmsfoolselfgeneratedtypographic,wang2024jailbreaklargevisionlanguagemodels,teng2025heuristicinducedmultimodalriskdistribution}. In these methods malicious information is embedded into pictures that could have been rejected through text input. Other attacks include using visual role play~\citep{ma2024visualroleplayuniversaljailbreakattack}, exploiting visual understanding capabilities~\citep{zou2024imagetotextlogicjailbreakimagination} and so on. We refer interested readers to a more comprehensive survey paper by \citet{ye2025surveysafetylargevisionlanguage}.

Attacks in robotics is not studied as extensively as language and vision generative models. Common attacks to visual inputs include gradient-based pixel-level attacks~\citep{du2021physicaladversarialattacksaerial,goodfellow2015explainingharnessingadversarialexamples} and patch-based attacks that can be realized in physical world~\citep{athalye2018synthesizingrobustadversarialexamples,xu2020adversarialtshirtevadingperson}. There are also recent works on attacking vision-language-action models~\citep{wang2025exploringadversarialvulnerabilitiesvisionlanguageaction}.

\section{Notation and Preliminaries}\label{section:prelim}
We denote individual vectors by letters $x,y$, and sequence vectors using $a,b,c$. When using subscript, $x_i,y_i$ mean the $i$-th entries, while $a_i,b_i,c_i$ mean their $i$-th elements. Symbol $\|\cdot\|$ represents the 2-norm of a vector or matrix. Symbol $\odot$ indicates entry-wise multiplication between vectors or matrices. Symbol $\oplus$ represents the direct sum (Cartesian Product) of linear spaces and we use it to define input spaces for sequences. 
For a positive integer $n\in\N^+$, we define $[n]=\{1,\cdots,n\}$. The identity matrix of dimension $d$ is represented by $I^d$.
For a set $\Omega\subset\R^d$, let $\partial\Omega$ be its boundary and $\overline{\Omega}$ be its closure. For a function $f:\R^d\to\R^d$ and a set $S\subset\R^d$, we denote $f(S)=\{f(x)|x\in S\}$. \emph{For most of this paper, we only discuss networks with the same input and output dimensions.}
\subsection{Neural Networks}
Next, we give an overview of common modern neural network building blocks. Let the input of the network be a vector $x\in\R^d$ where $d\in\N$ is the input dimension. The output $y\in\R^d$ is also a vector. One of the most elementary architectures is the Multi-Layer Perceptron.
\begin{definition}\label{def:mlp}
    An \emph{$m$-layer Multi-Layer Perceptron (MLP)} is a function $f:\R^d\to\R^d$ defined as
    \begin{align*}
        f(x)=\sigma_m(W_m\cdots\sigma_2(W_2\sigma_1(W_1x+\lambda_1)+\lambda_2)\cdots+\lambda_m)
    \end{align*}
    where $\{W_i\}_{i\in[m]}$ are trainable matrices, $\{\lambda_i\}_{i\in[m]}$ are trainable vectors called \emph{bias terms}, and $\{\sigma_i\}_{i\in[m]}$ are nonlinear entry-wise functions called \emph{activation functions}. Row dimensions of $W_1,\cdots,W_{m-1}$ are called \emph{hidden dimensions}.
\end{definition}
We define common examples of activation functions, namely ReLU, and GeLU~\citep{hendrycks2023gaussianerrorlinearunits} below:
\begin{align*}
    \mathrm{ReLU}(x)_i=\max\{x_i,0\},\mathrm{GeLU}(x)_i=x_i\cdot\frac12\Mp{1+\mathrm{erf}\Sp{\frac{x_i}{\sqrt2}}}.
\end{align*}
Here the subscript $i$ means the $i$-th entry of a vector, and $\mathrm{erf}(\cdot)$ is the Gaussian error function.

Residual connection~\citep{he2015deepresiduallearningimage} and layer normalization~\citep{ba2016layernormalization} are essential elements of modern deep neural networks that usually work together.
\begin{definition}\label{def:layernorm}
    \emph{Layer Normalization}\footnote{In real-world implementation a small $\varepsilon$ is added to the denominator in LN. Here we omit it for presentation simplicity and it does not change the proofs in this paper.} is a function $\LN:\R^d\to\R^d$ defined as
    \begin{align*}
        \LN(x)=\gamma\odot\frac{x-\overline{x}}{\|x-\overline{x}\|/\sqrt{d}}+\beta,\quad \overline{x}=\frac1d\sum_{i=1}^dx_i
    \end{align*}
    where $\gamma, \beta\in\R^d$ are trainable parameters. When used with residual connections, there are two variants called \emph{Pre-LayerNorm} and \emph{Post-LayerNorm}. When wrapped around a function $f:\R^d\to\R^d$, residual connection, Pre-LayerNorm and Post-LayerNorm are respectively defined as
    \begin{align*}
        g(x)=f(x)+x,\quad g(x)=f(\LN(x))+x,\quad g(x)=\LN(f(x)+x).
    \end{align*}
\end{definition}
Pre-LayerNorm has become common practice in modern neural networks as it stabilizes training~\citep{xiong2020layernormalizationtransformerarchitecture,openai2024gpt4technicalreport,touvron2023llama2openfoundation,chowdhery2022palmscalinglanguagemodeling}. Post-LayerNorm is also used sometimes~\citep{zhuo2025hybridnormstableefficienttransformer,li2024mixlnunleashingpowerdeeper,vaswani2023attentionneed}.

\paragraph{Sequence Models.} There are also specialized architectures dealing with sequential data, which take in an input sequence $a_1,\cdots,a_n$ and output another sequence $b_1,\cdots,b_n$. Attention~\citep{bahdanau2016neuralmachinetranslationjointly,vaswani2023attentionneed} is one of the most widely used sequence models.

\begin{definition}\label{def:attention}
    For trainable key, query, and value matrices $K,Q,V\in\R^{d\times d}$, a \emph{causally-masked attention layer} calculates outputs as\footnote{In practice, a scaling factor is inside the exp function~\citep{vaswani2023attentionneed}, which we omit, as it does not affect our analysis.}
    \begin{align*}
        b_i=\Attn(a)_i=\frac{1}{Z_i}\sum_{j=1}^i\exp(a_j^\top K^\top Qa_i)Va_j,\quad Z_i=\sum_{j=1}^i\exp(a_j^\top K^\top Qa_i).
    \end{align*}
\end{definition}
Causally-masked attention layers are usually used for autoregressive generation. Specifically, given input $a_1,\cdots,a_n$, the next token is generated from decoding $b_n$. After that this token is appended to the input and subsequent tokens are iteratively generated in the same way. There are also variants to $\Attn$ called linear attentions~\citep{yang2024gatedlinearattentiontransformers}. The simplest linear attention is RetNet~\citep{sun2023retentivenetworksuccessortransformer}.
\begin{definition}\label{def:retnet}
    For trainable parameters $K,Q,V\in\R^{d\times d}$, a \emph{Retention layer} calculates output as
    \begin{align}\label{eq:retnet}
        b_i=\Ret(a)_i=\sum_{j=1}^i\Sp{a_j^\top K^\top Qa_i}Va_j=S_iQa_i,\text{ where } S_i=S_{i-1}+Va_ia_i^\top K^\top.
    \end{align}
\end{definition}
In other words, $\Ret$ can be thought of as $\Attn$ without the non-linearity introduced by the soft-max function through $\exp$ and $Z_i$.
It admits a recurrent form as shown in \Cref{eq:retnet}, so autoregressive generation becomes faster. Other variants keep the recurrent form and use more complicated update rules for better performance. To keep the presentation clean, we defer the discussions on multi-head attention to the Appendix.

\paragraph{Transformer.} The Transformer architecture is widely used in many applications recently. Most of them can be expressed by the building blocks stated previously. To illustrate this, we take GPT-3 (\cite{brown2020languagemodelsfewshotlearners} referred to as GPT below) as an example here. A single block in GPT can be expressed as
\begin{align*}
    b_i=\TF(a)_i=W_2\mathrm{GeLU}(W_1\LN\Sp{c_i}+\lambda_1)+\lambda_2+c_i,  \text{ where }c_i=\Attn(\LN(a))_i+a_i.
\end{align*}

 Matrices $W_1,W_2^\top\in\R^{d\times d'}$. Besides, $\LN$ applying to a sequence means applying separately to each input vector, i.e. the layer norm of the $i$-th vector is $\LN(a)_i=\LN(a_i)$. In plain text, $\TF$ is an $\Attn$ followed by a two-layer MLP, each wrapped with Pre-LayerNorm, and GPT is a composition of several $\TF$s.

\subsection{Differential Topology}\label{section:topology}
Differential topology studies the properties of smooth manifolds that are invariant under smooth transformations. Since almost all neural networks are trained using back-propagation, the architectures are usually smooth.\footnote{Some building blocks, such as ReLU are not smooth everywhere. However they are only not smooth on zero measure sets, otherwise we cannot obtain gradient for a substantial amount of inputs. Topology can still tackle this scenario because such functions can always be approximated by smooth functions. We refer interested readers to standard references~\citep{hirsch1976differential} and omit this subtlety for simplicity.} Differential topology hence provides natural tools for us to prove surjectivity of neural networks.  In this section, we go over the necessary mathematical concepts and results that will be used later. We restrict our scope to smooth maps $f,g:\R^d\to\R^d$ here unless otherwise specified

One of the early triumphs of topology is the Brouwer's fixed point theorem.
\begin{theorem}[Brouwer's Fixed Point Theorem, {\cite[p.~117]{munkres1984elements}}]\label{thm:brouwer}
    Let $B^d(R)=\Bp{x\in\R^d\middle|\|x\|\leq R}$ be a $d$-dimensional ball with radius $R$. For every continuous function $f:B^d(R)\to B^d(R)$, there exists $x\in B^d(R)$ such that $f(x)=x$.
\end{theorem}
This theorem can be generalized from a ball $B^d(R)$ to any convex closed bounded set. {Most celebrated and common applications of Brouwer's fixed point theorem are in game theory, Economics, and the study of equilibria of dynamical systems.}

\begin{definition}
    \emph{Differential} $Df:\R^d\to\R^{d\times d}$ is defined as $\displaystyle Df(x)_{ij}=\partial f_i(x)/\partial x_j.$
\end{definition}
Since $f(x)$ can be approximated linearly in a small neighborhood around $x$, $Df(x)$ describes the local behavior of $f$ around $x$. More specifically $Df(x)_{ij}$ describes the rate the $i$-th dimension of output changes with regard to the $j$-th dimension of input.
Hence an invertible $Df(x)$ (equivalently one with $\det Df(x)\neq0$) indicates that $f$ behaves well around $x$
in the sense that no small neighborhood containing $x$ is collapsed to be lower dimensional after the application of function $f$. The next theorem, known as the Inverse Function Theorem, formalizes this intuition.
\begin{theorem}[Inverse Function Theorem]\label{thm:inverse}
    Let $x\in\R^d$ satisfy $\det Df(x)\neq0$, then there exist open sets $U\ni x, V\ni f(x)$, such that $f$ is bijective between $U$ and $V$.
\end{theorem}
One of the key concepts in algebraic topology is homotopy. {Roughly speaking, two {continuous} functions are homotopic if one can be continuously deformed to the other.}
\begin{definition}
    \emph{Homotopy} is a function class $\Bp{f_t:\R^d\to \R^d\middle |t\in[0,1]}$ such that the associated $F:\R^d\times[0,1]\to\R^d$, defined by $F(x,t)=f_t(x)$, is continuous. For two functions $f,g:\R^d\to \R^d$, they are \emph{homotopic} if there exists a homotopy such that $f_0=f,f_1=g$.
\end{definition}

Now we are ready to introduce Brouwer degree, a generalization of the idea presented in \Cref{thm:brouwer}. Degree theory is another powerful tool to prove surjectivity.
\begin{definition}\label{def:degree}
    (\cite[Definition 1.2.4]{DincaMawhin2021}) Let $\Omega\subset\R^d$ be an open bounded set {and $\overline{\Omega}$ be its closure. Let} $\overline{f}:\overline{\Omega}\to\R^d$ be the restriction of $f$ on $\overline{\Omega}$.
    {Then for any value $y\notin f(\partial \Omega)$ (i.e, any $y$ to which no $x$ on the boundary of $\Omega$ maps),}
    the Brouwer degree is defined by
    \begin{align}\label{eq:degree}
        \deg(f,\Omega,y)=\sum_{x\in {\overline{f}}^{-1}(y)}\mathrm{sgn}\det(Df(x)).
    \end{align}
\end{definition}
\begin{lemma}\label{lem:degree}
    (\cite[Theorem 1.2.2]{DincaMawhin2021}) If $f,g$ are homotopic, {$\Omega$ is an open bounded set}, and $v\notin F(\partial\Omega, t)$ for all $t\in[0,1]$, we have $\deg(f,\Omega,v)=\deg(g,\Omega,v)$.
\end{lemma}

    This property shows that Brouwer degree is homotopy invariant.
    This allows us to reduce the problem of calculating the degree of a complex function to that of a simpler one.
    Since nonzero degree implies that at least one term exists on the right hand side of \Cref{eq:degree}, it also implies the existence of a pre-image, which helps us prove surjectivity of complex functions.

\section{Surjectivity of Architectural Blocks in Modern Neural Networks}\label{section:theory}
In this section, we analyze the surjectivity of fundamental building blocks in modern neural networks using tools from differential topology. Before diving into the details, we formalize the setting. Though elementary, we start with the formal definition of surjectivity.
\begin{definition}
    A function $f:\X\to\Y$ is \emph{surjective}, if for any $y\in\Y$, there exists a {pre-image}, i.e., an $x\in\X$ such that $f(x)=y$.
\end{definition}
Surjectivity is closed under composition. Namely, if $f:\X\to\Y,g:\Y\to\Z$ are surjective, their composition $g\circ f$ is also surjective. This allows us to separately prove surjectivity of building blocks of neural networks and conclude that the whole network is surjective. It is often too good to hope that an architecture is always surjective with any parameter. For example, if we set $V=0$ in attention (\Cref{def:attention}), this layer will output nothing but zero. A less extreme example is when $V$ is not an invertible matrix, it is not possible to output a vector outside the subspace spanned by $V$'s column space. However, this almost never happens in practice because the set of non-invertible matrices takes up zero volume in the parameter space, and hence is almost never hit in trained models.
\begin{definition}\label{def:almost-always-surj}
    Let $\mathcal{H}_\Theta=\{f_\theta:\X\to\Y|\theta\in\Theta\}$ be a class of neural networks parameterized by $\theta\in\Theta$, where $\Theta$ is a subset of Euclidean space with Lebesgue measure and $\Y\subset\R^d$ has Lebesgue measure. $\mathcal{H}_\Theta$ is an \emph{almost always surjective} set if, except for $\theta$ in a zero-measure subset of $\Theta$, $f_\theta^{-1}(y)$ is nonempty for every $y\in\Y$ except for a zero-measure subset of $\Y$.
\end{definition}
Below we analyze which architectures are almost always surjective and in \Cref{section:safety} we discuss which real-world models are surjective when used in specific ways.

\subsection{Pre-LayerNorm}\label{sec:preln}
In this section, we prove that Pre-LayerNorm is surjective using a neat but nontrivial application of the Brouwer's fixed point theorem (\Cref{thm:brouwer}).
\begin{theorem}\label{thm:preln}
    Let $f:\R^d\to\R^d$ be a continuous function, then $g:\R^d\to\R^d$ defined by $g:x\mapsto f(\LN(x))+x$ is surjective.
\end{theorem}
\begin{proof}
    By the definition of Pre-LayerNorm, the normalized vector before the affine transformation has 2-norm $\sqrt{d}$. Therefore, $
    \|\LN(x)\|\leq \sqrt{d}\|\gamma\|+\|\beta\|$.
    Since $f$ is continuous on the compact ball of radius $\sqrt{d}\|\gamma\|+\|\beta\|$, we have
    \begin{align*}
        M=\sup_{x\in\R^d}\|f(\LN(x))\|<\infty.
    \end{align*}
      To prove surjectivity of $g$, we need to prove that for any $y\in\R^d$ there exists an input $x^*$ such that $g(x^*)=y$. Below we fix $y$ and prove such $x^*$ exists. 
        Let $F:\R^d\to\R^d$ be a function defined by
      $
        F(x)=y-f(\LN(x)).$
        Now we find a fixed point for $F$. Let $R=M+\|y\|+1$, then by triangle inequality we have $\|F(x)\|<R$. Therefore, $F$ maps $B^d(R)$ into itself. We can thus define the restriction of $F$ on $B^d(R)$ as a function $F|_{B^d(R)}:B^d(R)\to B^d(R)$. By \Cref{thm:brouwer}, we know that there exists $x^*\in B^d(R)$ such that $F|_{B^d(R)}(x^*)=x^*$. Plugging this $x^*$ back to the definition of $F$ we know $F(x^*)=x^*=y-f(\LN(x^*))$. Thus $g(x^*)=f(\LN(x^*))+x^*=y$. 
        This establishes that for  any $y\in\R^d$ there exists a corresponding input $x^*$, so $g$ is surjective.
\end{proof}

    This is a powerful theorem in the sense that it places minimal requirements on the function $f$. Nearly all modern neural networks are continuous, so the theorem implies that any architecture wrapped with Pre-LayerNorm is surjective. Notably, the proof does not rely on the specific expression of LayerNorm, but only on the fact that it is continuous and bounded. Hence if one uses other types of normalization functions with this property instead, like RMSNorm~\citep{zhang2019rootmeansquarelayer}, GroupNorm~\citep{wu2018groupnormalization} or DyT~\citep{zhu2025transformersnormalization}, this theorem still holds. Finally, this theorem can be easily extended to sequence models as we show below.
\begin{restatable}{theorem}{prelnseq}\label{thm:prelnseq}
    Let $f:\oplus_{i\in[n]}\R^d\to\oplus_{i\in[n]}\R^d$ be a continuous function, then $g:\oplus_{i\in[n]}\R^d\to\oplus_{i\in[n]}\R^d$ defined by $g(a)_i=f(\LN(a))_i+a_i$ is surjective.
\end{restatable}
The proof of \Cref{thm:prelnseq} is deferred to the Appendix.

Next, we briefly discuss Post-LayerNorm for completeness.
Since $\LN$ normalizes its input, Post-LayerNorm cannot be onto $\R^d$; ignoring constant inputs where the idealized $\LN$ without $\varepsilon$ is undefined, its image is
\begin{align*}
    S=\Bp{\gamma\odot z+\beta\middle|z\in\R^d,\ \|z\|=\sqrt d,\ \overline z=0}.
\end{align*}
It suffices for $f(x)+x$ to cover a neighborhood of 0: any nonconstant pre-image of $y\in S$ can be scaled into that neighborhood without changing its LayerNorm.
\begin{theorem}
    Let $f:\R^d\to\R^d$ be a bias-free GeLU MLP whose hidden dimensions are all $d$. Then $\LN(f(x)+x)$ is almost always surjective on $S$.
\end{theorem}
\begin{proof}
    Let $g(x)=f(x)+x$ and $s=\mathrm{GeLU}'(0)$. By the chain rule,
    \begin{align*}
        Dg(0)=sI^d\cdot W_m\cdot sI^d\cdots sI^d\cdot W_1+I=s^m\prod_{i=m}^1 W_i+I
    \end{align*}
    which is almost always full rank. Since $g(0)=0$, \Cref{thm:inverse} gives an open $U\ni0$ with $U\subset g(\R^d)$. For any $y\in S$, choose nonconstant $v$ with $\LN(v)=y$ and $\mu>0$ with $\mu v\in U$. Then $\LN(\mu v)=y$, and since $\mu v=g(x)$ for some $x$, we have $\LN(g(x))=y$.
\end{proof}
\begin{remark}
    We do not aim to provide a comprehensive result for MLPs with Post-LayerNorm. Many of the conditions in this proposition can be relaxed. The network architecture can be quite flexible, as long as its determinant almost never vanishes and the zero input maps to zero output. Moreover, MLPs without bias terms are not uncommon in practice~\citep{groeneveld2024olmoacceleratingsciencelanguage,touvron2023llama2openfoundation,chowdhery2022palmscalinglanguagemodeling}.
\end{remark}
This $S$-surjectivity does not compose directly, since the next block's domain would also become $S$; we leave $S$-to-$S$ Post-LayerNorm surjectivity for future work.

\subsection{Linear Attention}
In recent years, linear attention mechanisms have gained significant attention due to their improved scalability compared to the original attention mechanism. In this section, we prove that the Retention layer (\Cref{def:retnet}) is almost always surjective. The key step is to reduce the sequence-level problem to a token-level map of a particular form.

\begin{theorem}\label{thm:ret}
    $\mathrm{Ret}$ is almost always surjective.
\end{theorem}
\begin{proof}
    By definition, we need to prove that outside a zero-measure subset of parameters, every output sequence $b_1,\cdots,b_n$ except for a zero-measure subset of output space has a corresponding input sequence $a_1,\cdots,a_n$. The proof is by induction on the output sequence. The first output $b_1$ only depends on $a_1$:
    \begin{align*}
        b_1=a_1^\top K^\top Qa_1Va_1\Rightarrow a_1=\Sp{\Sp{V^{-1}b_1}^\top K^\top QV^{-1}b_1}^{-1/3}V^{-1}b_1.  \end{align*} 
    This solution implicitly assumes that $V$ is invertible and $\Sp{V^{-1}b_1}^\top K^\top QV^{-1}b_1\neq0$,  both excluding zero measure sets from the Euclidean space. The induction hypothesis is: when searching for a pre-image of $b_j$ for $j>1$, pre-images $a_1, \dots, a_{j-1}$ are already determined and our choice of $a_j$ can only depend on $b_j$. That is, we want to solve for $a_j$ that meets the following requirement:
    \begin{align*}
        b_j=S_{j-1}Qa_j+Va_ja_j^\top K^\top Qa_j,
    \end{align*}
    where $S_{j-1}$ is the recurrence used in \Cref{def:retnet} as a function of $a_1, \dots, a_{j-1}$.
    Since $V$ is invertible, let $a_j'=Va_j$ and  define $b_j$ as function of $a'_j$:
    \begin{align*}
        b_j=f\Sp{a_j'}=S_{j-1}QV^{-1}a_j'+a_j'a_j'^\top\Sp{V^{-1}}^\top K^\top QV^{-1}a_j'=Ma_j'+\Sp{a_j'^\top Na_j'}a_j',
    \end{align*}
    where matrices $M=S_{j-1}QV^{-1}, N=\Sp{V^{-1}}^\top K^\top QV^{-1}$. At this point, the sequence-level problem has been reduced to studying the map $f(x)=Mx+(x^\top Nx)x$. We invoke the following lemma, whose proof is given after the theorem.
    \begin{restatable}{lemma}{retlemma}\label{lem:ret}
        Let $f:\R^d\to\R^d$ be defined by $f(x)=Mx+(x^\top Nx)x$. Suppose $N$ ranges over full-rank matrices and $M$ ranges over either full-rank matrices or nonzero low-rank matrices.\footnote{When we say ``almost all'' low-rank matrices, we mean after restricting to the set of low-rank matrices. Within this set, the typical case has rank $d-1$.} Then $f$ is almost always surjective.
    \end{restatable}
    Outside a zero measure set of parameters, $N$ is full rank and $M$ is either full rank or nonzero low-rank. Applying \Cref{lem:ret} to function $f(x)=Mx+(x^\top Nx)x$ shows that there exists $a_j'$ with $f(a_j')=b_j$. Since $V$ is invertible, this gives the desired $a_j=V^{-1}a_j'$ and completes the induction. Thus $\Ret$ is almost always surjective.
\end{proof}
In summary, we first use the causal structure of $\mathrm{Ret}$ to reduce the sequence-level pre-image problem to a token-level problem that can be solved inductively, outside the zero-measure exceptional set for the first output. Second, after a change of variables, the token-level map takes the abstract form $f(x)=Mx+(x^\top Nx)x$. We then apply \Cref{lem:ret} to it. Now we show the proof of~\Cref{lem:ret}.

For convenience, we introduce the following concept before proving the lemma.
\begin{definition}
    A function $f:\R^d\to\R^d$ is \emph{proper} if for every sequence $\{x_k\}_{k\in\N}$ such that $\|x_k\|\to\infty$, we have $\|f(x_k)\|\to\infty$.
\end{definition}

\begin{proof}[Proof of \Cref{lem:ret}]
    We decompose the proof into two parts. First, we show that for the present function, properness is sufficient for surjectivity. This is formalized by \Cref{lem:ret-proper-surj-app}. Second, we show that properness holds for almost all matrices in the parameter class stated above, which is formalized by \Cref{lem:ret-almost-proper}. Combining these two lemmas proves the claim.
\end{proof}

We defer the proof of \Cref{lem:ret-proper-surj-app} to the appendix, and prove a slightly weaker version for the specific form of $f$ considered here to showcase how the degree theory tool from \Cref{section:topology} can be used.
\begin{lemma}\label{lem:ret-proper-surj}
    Let $f:\R^d\to\R^d$ be defined by $f(x)=Mx+(x^\top Nx)x$. Suppose $M$ is full rank, $f$ is proper, and $\det Df(x)\neq0$ for every $x\in f^{-1}(0)$. Then $f$ is surjective.
\end{lemma}
\begin{proof}
    Fix a target output $y\in\R^d$. By properness, there exists $R>0$ such that $\|f(x)\|>\|y\|$, for all $\|x\|=R$.
    Consider the homotopy
    \begin{align*}
        F(x,t)=f(x)-ty,\quad t\in[0,1].
    \end{align*}
    For every $x\in\partial B^d(R)$ and every $t\in[0,1]$, we have $F(x,t)\neq0$, because otherwise $f(x)=ty$ would imply $\|f(x)\|\leq\|y\|$, contradicting the choice of $R$. Therefore, by the homotopy invariance of Brouwer degree in \Cref{lem:degree},
    \begin{align*}
        \deg(f-y,B^d(R),0)=\deg(f,B^d(R),0).
    \end{align*}
    It remains to calculate $\deg(f,B^d(R),0)$. By the non-degeneracy assumption, we can use the explicit formula in \Cref{eq:degree}. The origin is always a pre-image of $0$. Moreover, $Df(x)=M+(x^\top Nx)I^d+x\Sp{(N+N^\top)x}^\top$, and therefore $Df(0)=M$.
    Since $M$ is full rank, the contribution from $x=0$ is $\mathrm{sgn}\det M$. Now take any nonzero pre-image $x$ of $0$. The explicit form of $f$ gives $f(-x)=-f(x)$, so $-x$ is also a pre-image of $0$. These two pre-images are distinct, and the expression for $Df$ gives $Df(-x)=Df(x)$. Thus they contribute the same sign in \Cref{eq:degree}. Therefore the nonzero pre-images contribute an even integer, and
    \begin{align*}
        \deg(f,B^d(R),0)=\mathrm{sgn}\det M+2k
    \end{align*}
    for some $k\in\mathbb{Z}$. In particular, $\deg(f,B^d(R),0)\neq0$.

    It follows that $\deg(f-y,B^d(R),0)\neq0$. By the defining property of Brouwer degree, there exists $x\in B^d(R)$ such that $f(x)-y=0$. Since $y$ was arbitrary, $f$ is surjective.
\end{proof}
The intuition behind this degree calculation is that, when studying roots of a smooth function, local continuous changes should not change the existence of a root as long as no root crosses the boundary of the region under consideration. Brouwer degree formalizes this intuition. Properness lets us choose a sufficiently large ball whose boundary is not involved in the target equation, and homotopy invariance lets us replace $f-y$ by the more convenient map $f$ for the purpose of calculating degree. For this particular function, the calculation at $0$ is transparent because $Df(0)=M$. We provide more applications of Brouwer degree in the Appendix.

\begin{restatable}{lemma}{retalmostproper}\label{lem:ret-almost-proper}
    Let $f:\R^d\to\R^d$ be defined by $f(x)=Mx+(x^\top Nx)x$. Suppose $N$ ranges over full-rank matrices and $M$ ranges over either full-rank matrices or nonzero low-rank matrices. Then $f$ is proper for almost all such choices of $M,N$.
\end{restatable}
The proof of \Cref{lem:ret-almost-proper} is deferred to the Appendix. It is natural to ask whether the original attention mechanism is also surjective. We show that it is not through the following theorem.
\begin{restatable}{theorem}{notsur}\label{thm:notsur}
    $\mathrm{Attn}$ is not almost always surjective.
\end{restatable}
We defer the proof of this theorem to the Appendix. Here we provide a brief overview of why attention itself is not surjective.
For a positive-measure set of parameters, the symmetric part of $K^\top Q$ is negative definite. If the first output $b_1$ is kept in a bounded region, then $a_1$ is also bounded when $V$ is invertible. At the second position, the self-attention weight on $Va_2$ decays like $\exp(a_2^\top K^\top Qa_2)$ as $\|a_2\|$ grows, while the cross term involving $a_1$ grows only linearly in $\|a_2\|$. Consequently $b_2$ remains bounded whenever $b_1$ is bounded, so output pairs with bounded $b_1$ and sufficiently large $b_2$ form a positive-measure unreachable set.
We provide more such negative results in the Appendix.

\section{Implications of Surjectivity on Modern Networks and Applications}\label{section:safety}
In this section, we discuss how our theoretical results on the surjectivity of large models from \Cref{section:theory} relate to safety concerns in real-world settings, particularly with respect to adversarial attacks. Our goal is not to provide an exhaustive list of all practical models or attacks. 
Instead, we present concrete examples of generative models used across diverse application areas --- such as language, vision, and robotics --- to illustrate the scope and implications of the surjectivity theory.

\subsection{Language Models}
By \Cref{thm:preln,thm:prelnseq}, Pre-LayerNorm is surjective; since it is used in every layer of modern Transformers, the Transformer architecture itself is also surjective.
\begin{corollary}\label{cor:transformer}
    $\TF$ is almost always surjective, and so are compositions of $\TF$.
\end{corollary}

Our results are proved in embedding space and assume direct control over the input embedding, unlike typical LLM deployments. Language models take discrete tokens mapped through a fixed embedding function, and most generative models are decoder-only and autoregressive. Thus our theorems do not directly apply to prompt-based attacks that exploit autoregressive generation~\citep{zou2023universaltransferableadversarialattacks}. This limitation is arguably reassuring (\Cref{thm:autoregressive-transformer-not-surj}): if surjectivity held in that setting, it could imply a broader vulnerability surface.

However, our results raise questions about how to interpret language-model input-output behavior. Copyright concerns often cite a model's ability to reproduce specific outputs, such as a sentence from a proprietary source~\citep{he2025fantasticcopyrightedbeastsnot}, as evidence.
Our findings imply that, in principle, any sentence can be produced by decoding the final-layer output from a suitable input embedding, even if the model was trained autoregressively. In the Appendix, we demonstrate this on GPT-2~\citep{Radford2018ImprovingLU}: decoding the last layer yields a 37-word sentence from a 2025 New York Times article that could not have appeared in training. Thus, caution is needed when inferring private or copyrighted data from outputs alone, especially for inputs outside typical usage patterns.

\subsection{Vision Models}\label{sec:vision}
Diffusion models are a class of generative models originally proposed for image generation~\citep{sohldickstein2015deepunsupervisedlearningusing}. They are now widely applied to other domains, including video generation~\citep{liu2024sorareviewbackgroundtechnology}, robotics~\citep{chi2024diffusionpolicyvisuomotorpolicy}, and beyond.
Here we describe the generation process of diffusion models. Let the data we want to generate be represented by a vector $x\in\R^d$. First, one generates $x(0)\sim\mathcal{N}\Sp{0,I^d}$ from a Gaussian distribution. We regard $x$ as a variable depending on $t\in[0,1]$. After that, $x$ evolves, or diffuses, according to a velocity vector field $v(x,t)\in\R^d$. The field $v$ is trained in the hope that $x(1)$ follows the same distribution as the data distribution $p(x)$ that we want to generate. Formally we have
\begin{align*}
    \dd x/\dd t=v(x,t), x(0)\sim\mathcal{N}\Sp{0,I^d},x(1)\sim p(x).
\end{align*}
In practice directly solving this equation is often intractable, so we discretize $[0,1]$ into intervals $\Bp{\Mp{z_k,z_{k+1}}}_{k\in[m-1]}$ with $z_1=0,z_m=1$ and generate via approximation
\begin{align}\label{eq:diffusion}
    x\Sp{z_{k+1}}=x(z_k)+v(x(z_k),z_k)(z_{k+1}-z_{k}), \text{for all } k\in[m-1].
\end{align}

Early approaches to diffusion models inject Gaussian noise at inference time~\citep[see e.g.][]{ho2020denoisingdiffusionprobabilisticmodels}, resulting in stochastic updates that led to high quality outputs but slowed down the generation process. Here, we focus on the more recent \emph{deterministic} ODE solvers~\citep{song2022denoisingdiffusionimplicitmodels} as described above, which also benefit from faster generation. Velocity predictor $v$ is usually parameterized as a U-Net~\citep{ronneberger2015unetconvolutionalnetworksbiomedical,ho2020denoisingdiffusionprobabilisticmodels} or Transformer~\citep{peebles2023scalablediffusionmodelstransformers}. When $v$ is implemented using a transformer, $x$ is tokenized into a sequence of vectors before being passed to the Transformer.

For image generation, a noisy image $x(0)$ evolves under $v$ into a noiseless image $x(1)$. By \Cref{eq:diffusion}, since the first layer of $v$ is typically a normalization (GroupNorm for U-Nets and Pre-LayerNorm for Transformers), \Cref{thm:preln,thm:prelnseq} imply that deterministic diffusion models are almost always surjective from noise space to output space, suggesting inherent vulnerability to adversarial attacks. 
Indeed, prior work~\citep{zeng2024advi2iadversarialimageattack} constructed noise vectors $x(0)$ that generate harmful content; our results show that, regardless of training, such an input exists for any output image. 

\subsection{Robotics}
Neural networks, in particular sequence models, are increasingly common in robotics, and are making robots increasingly powerful. However, this also gives rise to safety concerns given our results.
As an example, let us take a policy network from prior work~\citep{radosavovic2023realworldhumanoidlocomotionreinforcement}, which follows a widely used design in practice.
The network is implemented by a causally masked Transformer, i.e. compositions of $\mathrm{TF}$. 
At timestep $t$, the action $b_t$\footnote{Robotics convention uses $a$ for action, but we use $b$ here to keep notations in this paper consistent.} is generated by the Transformer on input sequence $a_1,b_1,\cdots,a_{t-1},b_{t-1},a_t$. Here $a$ is the sequence of observations from the environment.

Note that this sequence model diverges from the ones we studied in the earlier section, by interleaving the true input sequence (observations) and previous outputs (actions). This interleaving is done to improve the smoothness of robot actions.
Using similar techniques from \Cref{section:theory}, we can prove that this policy network is almost always surjective.
\begin{restatable}{theorem}{rob}\label{thm:label}
    Let $\mathrm{Rob}$ be a composition of $\TF$ blocks. Given a sequence $a$, we iteratively calculate sequence $b$ as $b_t=\mathrm{Rob}(a_1,b_1,\cdots,b_{t-1},a_t),t\geq2;b_1=\mathrm{Rob}(a_1)$.
    This defines a function $f$ from $a$ to $b$. $f$ is almost always surjective.
\end{restatable}

We defer the proof to the Appendix. This theorem means that for the policy network described in~\cite{radosavovic2023realworldhumanoidlocomotionreinforcement},
almost any action sequence can be induced by some corresponding sequence of inputs --- such as a video clip played for the robot --- regardless of how undesirable or unsafe the resulting behavior may be.

\section{Broader Discussion on the Implications of Surjectivity.}
\label{sec:broad}
In this paper, we introduced the study of surjectivity of neural networks as a concrete formalization of studying the power of safety training and jailbreak vulnerabilities.
In this section, we dive into the broader discussion of what surjectivity is, what it is not, and what surjectivity of modern networks implies for AI safety research.

\subsection{Theoretical Implications and the Formalism of Surjectivity}
\label{sec:broad:theory}

\paragraph{On Theoretical Implications of Surjectivity.}
One major goal of safety training is to limit a model's ability to generate harmful outcomes.
From a theoretical perspective, surjectivity implies that a model is vulnerable to jailbreaks in principle. That is, every outcome including those that are considered harmful by the model providers, can be generated by some input. 

We make no claim in the other direction. In particular, it is possible that many or even every harmful behavior by a non-surjective model can still be elicited by some input, while the model's lack of surjectivity is due to its inability to produce other non-harmful behavior.

\paragraph{On Surjectivity versus Model Capability.}
Model capability perhaps is best captured through details of the input-output behavior. 
By design, our study of surjectivity is agnostic to whether the input-output relationship a model captures is a complex or an interesting one. 
Indeed, even simple functions, such as the \emph{identity} function, can be surjective, while they are often not capturing interesting input-output behaviors. 
Our study of surjectivity intentionally decouples risks that are rooted in an attacker's ability to elicit particular behaviors --- which is the main consideration of jailbreaks --- from 
domain-specific risks that arise from having highly capable AI models in certain areas (such as bioweapons, etc.) in the first place.

One can take for granted that models that have undergone safety-training (and as a result are the main subject of the study of jailbreaks) are already highly capable models that have been trained on troves of relevant data to capture the complex relationship between the input and output spaces that deviates significantly from simple functions such as the identity. By establishing that these trained models are almost always surjective, our work highlights their inherent vulnerability, regardless of how capable the models are. That is, with enough information and computational power, an attacker can elicit any behavior from the model, including harmful ones.

\paragraph{On Surjectivity Versus Having Full Support.}
One might ask how surjectivity is different from a common assumption that generative models have full support?
This question stems from viewing the outcome of the neural network as a stochastic function from input to the output space, while in this work we view fully trained networks as deterministic functions. 
When considering fully trained generative models as deterministic functions, we find two perspectives to be instructive. 

The first perspective is to consider generative models with deterministic decoding, e.g., decoding the probability distribution greedily, with beam search, or with temperatures close to $0$. In this perspective, the observation that a stochastic function has full support does not imply that its deterministic decoding schemes are surjective.

A second perspective is to consider the hidden embedding computed by the network as a deterministic function.
Take GPT, for example, as a function that computes vector $b$, the hidden embedding output of its last block.
Embedding $b$ is used by the model to output any token $i$ from a finite set of tokens with embeddings $h_i$, with probability $p(i\mid b)\propto \exp(b^\top h_i)$. 
GPT has full support, i.e., $p(i\mid b)>0$ for all $i$, for any finite hidden embedding $b$.
On the other hand, our surjectivity results can make a particular token arbitrarily likely when its embedding has a direction that separates it from the other token embeddings. Take an arbitrary token $i$ with this property. Surjectivity of the network implies that for some input, the transformer's hidden embedding is $b = \lambda h_i$ for large $\lambda$. If $h_i^\top h_i>h_i^\top h_j$ for all other token embeddings $h_j$, then $p(j|b)\rightarrow 0$ for all $h_j \neq h_i$ as $\lambda\to\infty$. 
Thus surjectivity of the transformer implies that such a token can be made the deterministic greedy-decoding output.

\paragraph{On Surjectivity versus Impact of Implicit Regularization.}
There is a line of work~\citep{simon2018algorithmic,arora2019implicit} suggesting that standard optimization algorithms such as gradient descent provide implicit regularization, encouraging the models to bias towards lower-rank parameters. A natural question is whether such structure would impact our analysis on being almost always surjective. We would like to remark that the ``low-rank'' phenomena typically correspond to approximately low-rank matrices whose singular values decay but are not exactly zero in practice. These matrices are still full rank, and hence the corresponding linear maps have range equal to the full output space. From the perspective of our results, such models remain in the generic region where surjectivity holds; the main effect of regularization is to make certain output directions correspond to small singular directions, so that reaching them may require large-norm or highly atypical inputs. This is precisely the regime in which adversarial optimization can search for off-distribution inputs that exploit these directions, which is consistent with the ``in-principle'' vulnerability captured by surjectivity.

\subsection{Practical Implications of Surjectivity in Safety}

\paragraph{On computational and statistical considerations.}
In practice, the relationship between surjectivity and jailbreak vulnerabilities is more nuanced than explored in Section~\ref{sec:broad:theory}.
Surjectivity is an existential rather than a constructive statement about the mapping between input and output spaces: the existence of a pre-image $x$ for a harmful output $y$ does not imply that such a pre-image can be found computationally or information-theoretically efficiently.

From a computational perspective, finding such a pre-image may be intractable in the worst case. However, worst-case hardness results do not provide meaningful safety guarantees, since attackers need only succeed on some harmful instances and may have access to significant computational resources. The proliferation of jailbreaks in practice suggests that computational difficulty is rarely a bottleneck in the average or practical case.

From an information-theoretic perspective, surjectivity-based risks depend on the attacker's knowledge of the target outcome $y$. For example, in the language domain, a model's surjectivity demonstrates moderate risk, as it could point to \emph{repeated-after-me} attacks. Though this attack would let the model output harmful content, it requires that the attacker knows the harmful text they want to elicit in detail. They cannot elicit sensitive hidden information
that is a priori unavailable to the attacker (e.g., personally identifiable information, bioweapon
risks). In other domains, such as robotics, the risk is more severe. For example, when outcome $y$ represents an action in the physical domain --- such as the trajectory of an autonomous drone or a robot arm --- the knowledge of the exact trajectory a drone must take to hit an object at a destructive speed may be known to the attacker. In these settings, surjectivity presents a significant concern as an attacker with knowledge of $y$ can craft an input that elicits the destructive behavior.

Still, many jailbreaks that do uncover new information can be framed in terms of surjectivity. For example, \emph{suffix-injection} attacks~\citep{wei2023jailbrokendoesllmsafety,zou2023universaltransferableadversarialattacks,wang2024attngcgenhancingjailbreakingattacks} work by eliciting outputs of the form $y = ab$, where $a$ is a fixed prefix (e.g., ``Certainly! Here is ...'', ``Sure, the answer is ...'') and $b$ contains novel potentially harmful information the model providers have sought to prohibit.
Here the attacker's goal is not to learn anything from $a$ itself, but to ensure that the model commits to $a$ as the starting condition, exploiting the autoregressive continuation that makes generating $b$ more likely.
In other words, even though $a$ carries no new information, the ability to force such a prefix (as captured through the lens of surjectivity) significantly increases the probability of eliciting harmful $b$.

More broadly, we believe that a natural direction for future work is to examine how much partial information about $y$ is sufficient to reconstruct a problematic input $x$ --- for instance, quantifying what fraction of the output tokens must be fixed or how many adaptive queries an attacker can make before finding an $x$ that produces an output in the vicinity of $y$.

\paragraph{On Implications of Surjectivity on Safety Interventions.}
Surjectivity also bears on the discourse in the AI safety community, in particular about two types of overall practical approaches to safety, which we summarize as ``train-for-safety''~\citep{grattafiori2024llama3herdmodels,NEURIPS2024_9aa51796} versus ``filter-for-safety''~\citep{inan2023llamaguardllmbasedinputoutput,shi2025promptarmor}.
The former includes works on post-training for safety, RLHF/RLAIF, harmless finetuning, and safety pre-training that aim to bake restrictions on the output space into model weights at training time, while the latter includes task-specific filters and constitutional classifiers that post-hoc filter the generated outcomes. The train-for-safety methods offer low-latency and efficient generation policies that are highly desirable to the model providers compared to post-hoc filtering that discards harmful generated responses. The surjectivity of a model can be taken as further evidence that the train-for-safety paradigm is not a sufficient line of defense on its own.

\section*{Acknowledgment}
We thank Zhiyuan Li, Kaifeng Lyu, Yen Jen Wang, Jiaxin Ge, Andrew Wagenmaker, Margalit Glasgow, Haiwen Feng, Junyi Zhang, Surbhi Goel, and Ariel Procaccia for helpful conversations. This work was supported in part by the National Science Foundation under grant CCF-2145898, by the Office of Naval Research under grant N00014-24-1-2159, by an Alfred P.\ Sloan fellowship, and by a Schmidt Sciences AI2050 fellowship.
\bibliographystyle{plainnat}
\bibliography{reference.bib}

\newpage
\appendix

\section{\texorpdfstring{Omitted Proof in \Cref{section:theory} and \Cref{section:safety}}{Omitted Proof in Section}}
\prelnseq*
\begin{proof}
    The proof is very similar to that of \Cref{thm:preln}. Let
    \begin{align*}
        M=\sup_{a\in\oplus_{i\in[n]}\R^d}\|f(\LN(a))\|<\infty.
    \end{align*}
    The finiteness of $M$ follows because each $\LN(a_i)$ has 2-norm at most $\sqrt{d}\|\gamma\|+\|\beta\|$, so $\LN(a)$ lies in a compact ball in $\oplus_{i\in[n]}\R^d$ on which $f$ is bounded.
    For any specific output sequence $b\in\oplus_{i\in[n]}\R^d$ for which we want to find a corresponding input, construct $R=M+\|b\|+1$. Then applying \Cref{thm:brouwer} to the function $F(a)=b-f(\LN(a))$ restricted to $B^{nd}(R)$ proves the existence of a corresponding input. Hence $g$ is surjective.
\end{proof}
Like \Cref{thm:preln}, this proof can also be extended to other normalizations that are continuous and have bounded output.

\label{sec:odd-proper}
\begin{lemma}\label{lem:proper-odd-surj}
    Let $f:\R^d\to\R^d$ be continuous, proper, and satisfy $f(-x)=-f(x)$. Then $f$ is surjective.
\end{lemma}
\begin{proof}
    Fix a target output $y\in\R^d$. By properness, there exists $R>0$ such that
    \begin{align*}
        \|f(x)\|>\|y\|,\quad\text{for all }\|x\|=R.
    \end{align*}
    Let $\Omega=\{x\in\R^d\mid \|x\|<R\}$ and consider the homotopy
    \begin{align*}
        F(x,t)=f(x)-ty,\quad t\in[0,1].
    \end{align*}
    For every $x\in\partial\Omega$ and every $t\in[0,1]$, $F(x,t)\neq0$. Hence \Cref{lem:degree} gives
    \begin{align*}
        \deg(f-y,\Omega,0)=\deg(f,\Omega,0).
    \end{align*}
    It remains to show that $\deg(f,\Omega,0)\neq0$. Define the normalized boundary map
    \begin{align*}
        \psi_R:S^{d-1}\to S^{d-1},\quad \psi_R(u)= \frac{f(Ru)}{\|f(Ru)\|}.
    \end{align*}
    This is well-defined because $0\notin f(\partial\Omega)$. We denote by $\deg_S(\psi_R)$ the degree of this map between oriented spheres, in the sense of \cite[Section~2.2]{hatcher2002algebraic}. By the boundary definition of Brouwer degree and the normalization property of the boundary map \cite[Definition~1.2.5 and Proposition~1.2.6]{DincaMawhin2021}, applied to the ball $\Omega$, we have
    \begin{align*}
        \deg(f,\Omega,0)=\deg_S(\psi_R);
    \end{align*}
    here the identification $S^{d-1}\ni u\mapsto Ru\in\partial\Omega$ preserves the standard boundary orientation.
    Since $f(-x)=-f(x)$, the map $\psi_R$ satisfies $\psi_R(-u)=-\psi_R(u)$. Hence, by \cite[Proposition~2B.6]{hatcher2002algebraic}, $\deg_S(\psi_R)$ is an odd integer. Therefore $\deg(f,\Omega,0)\neq0$.

    It follows that $\deg(f-y,\Omega,0)\neq0$, and the defining property of Brouwer degree implies that there exists $x\in\Omega$ such that $f(x)=y$. Since $y$ was arbitrary, $f$ is surjective.
\end{proof}

\begin{lemma}\label{lem:ret-proper-surj-app}
    Let $f:\R^d\to\R^d$ be defined by $f(x)=Mx+(x^\top Nx)x$. If $f$ is proper, then $f$ is surjective.
\end{lemma}
\begin{proof}
    The function $f$ is continuous and satisfies $f(-x)=-f(x)$. The claim follows from \Cref{lem:proper-odd-surj}.
\end{proof}

\retalmostproper*
\begin{proof}
    Write $x=ru$ with $r=\|x\|$ and $u\in S^{d-1}$, and let $q(u)=u^\top Nu$. Then
    \begin{align*}
        f(ru)=rMu+r^3q(u)u.
    \end{align*}
    If $|q(u)|$ is bounded away from zero, then the cubic term dominates the linear term as $r\to\infty$. Therefore, a failure of properness can only occur along directions approaching set
    \begin{align*}
        Z_N=\{u\in S^{d-1}|u^\top Nu=0\}.
    \end{align*}
    Consider a sequence $r_k u_k$ with $r_k\to\infty$ and $u_k\to u\in Z_N$. Since the term $r_k^3q(u_k)u_k$ is in the radial direction $u_k$, it cannot cancel the component of $r_kMu_k$ orthogonal to $u_k$. Hence, if $\|f(r_k u_k)\|$ stays bounded, we must have
    \begin{align*}
        Mu\in\mathrm{span}\{u\}.
    \end{align*}
    Thus a necessary condition for non-properness is
    \begin{align}\label{eq:proper-condition}
        Z_N\cap \{u\in S^{d-1}|Mu\in\mathrm{span}\{u\}\}\neq\varnothing.
    \end{align}
    Equivalently, if the intersection in \Cref{eq:proper-condition} is empty, compactness of $Z_N$ gives a uniform lower bound on the orthogonal component of $Mu$ near $Z_N$, while away from $Z_N$ the cubic term dominates. In this case, $f$ is proper.

    The condition $Mu\in\mathrm{span}\{u\}$ imposes $d-1$ linear constraints for each fixed $u$. Since $Z_N$ is at most $(d-2)$-dimensional for almost all $N$, the union of these exceptional choices of $M$ has measure zero in the full-rank matrix space. The same dimension count applies on the generic rank-$(d-1)$ stratum of the low-rank matrices. The lower-rank cases form a smaller zero measure subset inside the set of low-rank matrices. Thus, for almost all choices of full-rank $N$ and such $M$, the intersection in \Cref{eq:proper-condition} is empty, and therefore $f$ is proper.
\end{proof}

\notsur*
\begin{proof}
    It suffices to consider sequences of length two. We exhibit a positive-measure set of parameters for which a positive-measure set of output pairs cannot be reached.
    Let $A=K^\top Q$, and consider parameters for which $V$ is invertible and the symmetric part $(A+A^\top)/2$ is negative definite. This set has positive measure because it contains the product of an open neighborhood of $(K,Q)=(I^d,-I^d)$ with the open set of invertible $V$.

    Fix such parameters. There exists $c>0$ such that
    \begin{align*}
        u^\top A u\leq -c\|u\|^2,\quad\text{for all }u\in\R^d.
    \end{align*}
    Fix $B>0$. Suppose an output pair $(b_1,b_2)$ with $\|b_1\|\leq B$ has a pre-image $(a_1,a_2)$. Since the first attention output is $b_1=Va_1$, we have
    \begin{align*}
        \|a_1\|\leq \|V^{-1}\|B.
    \end{align*}
    For the second output,
    \begin{align*}
        b_2
        =
        \frac{\exp(a_1^\top Aa_2)Va_1+\exp(a_2^\top Aa_2)Va_2}
        {\exp(a_1^\top Aa_2)+\exp(a_2^\top Aa_2)}.
    \end{align*}
    Write $r=a_1^\top Aa_2$, $q=a_2^\top Aa_2$, and
    \begin{align*}
        \lambda=\frac{\exp(q)}{\exp(r)+\exp(q)}.
    \end{align*}
    Then $b_2=(1-\lambda)b_1+\lambda Va_2$. Since
    \begin{align*}
        q\leq -c\|a_2\|^2,\qquad |r|\leq \|A\|\|V^{-1}\|B\|a_2\|,
    \end{align*}
    we have
    \begin{align*}
        \lambda\|Va_2\|
        \leq
        \|V\|\|a_2\|\exp\Sp{-c\|a_2\|^2+\|A\|\|V^{-1}\|B\|a_2\|}.
    \end{align*}
    The right-hand side is bounded uniformly over $a_2$; call this bound $M_B$. Therefore every reachable output pair with $\|b_1\|\leq B$ must satisfy
    \begin{align*}
        \|b_2\|\leq B+M_B.
    \end{align*}
    Hence the positive-measure set
    \begin{align*}
        \Bp{(b_1,b_2)\in\R^d\oplus\R^d\middle|\|b_1\|<B,\ B+M_B+1<\|b_2\|<B+M_B+2}
    \end{align*}
    is disjoint from the image of the length-two attention map. For longer sequences, taking the Cartesian product with any positive-measure set of remaining outputs gives the same obstruction. Thus attention is not surjective, even up to a zero-measure exceptional subset of outputs, on a positive-measure set of parameters. Therefore $\Attn$ is not almost always surjective.
\end{proof}

\rob*
\begin{proof}
    The proof is by induction and resembles the proof of \Cref{thm:ret}. Output $b_1$ only depends on $a_1$. By \Cref{cor:transformer} we know that $b_1=\mathrm{Rob}(a_1)=\TF(a_1)$ is surjective. When constructing $a_j$, we assume that all $a_1,\cdots,a_{j-1}$ have already been determined by $b_1,\cdots,b_{j-1}$. In this way, $b_j$ only depends on $a_j$, and the dependence is through a function that is a composition of functions with Pre-LayerNorms. By \Cref{thm:preln} we know that this function is surjective. In conclusion, $\mathrm{Rob}$ is surjective.
\end{proof}

\section{Additional Results}

\subsection{Negative Results on Surjectivity}

\begin{theorem}\label{thm:relu-mlp-not-surj}
    Let $d_1>d$ and consider the class of two-layer ReLU MLPs $f:\R^d\to\R^d$ of the form
    \begin{align*}
        f(x)=W_2\mathrm{ReLU}(W_1x+\lambda_1)+\lambda_2,
    \end{align*}
    where $W_1\in\R^{d_1\times d}$, $W_2\in\R^{d\times d_1}$, $\lambda_1\in\R^{d_1}$, and $\lambda_2\in\R^d$. This class is not almost always surjective.
\end{theorem}
\begin{proof}
    We exhibit a positive measure set of parameters for which $f$ is not surjective. Let
    \begin{align*}
        \Omega=\left\{(W_1,W_2,\lambda_1,\lambda_2)\mid (W_2)_{1j}>0\text{ for every }j\in[d_1]\right\}.
    \end{align*}
    This set has positive Lebesgue measure in the parameter space. For any parameters in $\Omega$ and any $x\in\R^d$, every coordinate of $\mathrm{ReLU}(W_1x+\lambda_1)$ is nonnegative. Therefore the first coordinate of the output satisfies
    \begin{align*}
        f(x)_1
        =
        \sum_{j=1}^{d_1}(W_2)_{1j}\mathrm{ReLU}(W_1x+\lambda_1)_j+(\lambda_2)_1
        \geq
        (\lambda_2)_1.
    \end{align*}
    Hence no vector $y\in\R^d$ with $y_1<(\lambda_2)_1$ belongs to the image of $f$. Thus every network with parameters in $\Omega$ is not surjective. Since $\Omega$ has positive measure, the non-surjective parameter set is not measure zero, and the class is not almost always surjective.
\end{proof}

We can prove a similar negative result for MLPs with GeLU activation.
\begin{theorem}\label{thm:gelu-mlp-not-surj}
    Let $d_1>d$ and consider the class of two-layer GeLU MLPs $f:\R^d\to\R^d$ of the form
    \begin{align*}
        f(x)=W_2\mathrm{GeLU}(W_1x+\lambda_1)+\lambda_2,
    \end{align*}
    where $W_1\in\R^{d_1\times d}$, $W_2\in\R^{d\times d_1}$, $\lambda_1\in\R^{d_1}$, and $\lambda_2\in\R^d$. This class is not almost always surjective.
\end{theorem}
\begin{proof}
    Let
    \begin{align*}
        m_{\mathrm{GeLU}}=\inf_{t\in\R}\mathrm{GeLU}(t).
    \end{align*}
    This quantity is finite because scalar GeLU is continuous, tends to $0$ as $t\to-\infty$, and tends to $\infty$ as $t\to\infty$.

    As in the ReLU case, consider the positive measure set of parameters
    \begin{align*}
        \Omega=\left\{(W_1,W_2,\lambda_1,\lambda_2)\mid (W_2)_{1j}>0\text{ for every }j\in[d_1]\right\}.
    \end{align*}
    For any parameters in $\Omega$ and any $x\in\R^d$, each coordinate of $\mathrm{GeLU}(W_1x+\lambda_1)$ is at least $m_{\mathrm{GeLU}}$. Therefore
    \begin{align*}
        f(x)_1
        =
        \sum_{j=1}^{d_1}(W_2)_{1j}\mathrm{GeLU}(W_1x+\lambda_1)_j+(\lambda_2)_1
        \geq
        m_{\mathrm{GeLU}}\sum_{j=1}^{d_1}(W_2)_{1j}+(\lambda_2)_1.
    \end{align*}
    Thus the first coordinate of $f(x)$ is uniformly lower bounded over all $x\in\R^d$. Any vector $y\in\R^d$ with
    \begin{align*}
        y_1<m_{\mathrm{GeLU}}\sum_{j=1}^{d_1}(W_2)_{1j}+(\lambda_2)_1
    \end{align*}
    is not in the image of $f$. Hence every network with parameters in $\Omega$ is not surjective. Since $\Omega$ has positive measure, the class is not almost always surjective.
\end{proof}

\begin{theorem}\label{thm:autoregressive-transformer-not-surj}
    Let $\mathcal{T}_\theta$ be any finite composition of $\mathrm{TF}$. For any prompt length $n\in\N^+$, define the two-step autoregressive map
    \begin{align*}
        \mathrm{AR}_{\theta,n}:\oplus_{i\in[n]}\R^d\to\R^d\oplus\R^d
    \end{align*}
    as follows. Given input sequence $a=(a_1,\cdots,a_n)$, let
    \begin{align*}
        b_1=\mathcal{T}_\theta(a_1,\cdots,a_n)_n,\qquad
        b_2=\mathcal{T}_\theta(a_1,\cdots,a_n,b_1)_{n+1},
    \end{align*}
    and set $\mathrm{AR}_{\theta,n}(a)=(b_1,b_2)$. The class of maps $\mathrm{AR}_{\theta,n}$ is not almost always surjective. In fact, no such $\mathrm{AR}_{\theta,n}$ is surjective.
\end{theorem}
\begin{proof}
    The central observation is that the residual connection passes the input at each position to the output at the same position, while the non-residual terms are evaluated after LayerNorm and are therefore bounded. Recall from \Cref{section:prelim} that one Transformer block takes the form
    \begin{align*}
        \TF(x)_i
        =
        x_i+\Attn(\LN(x))_i
        +
        W_2\mathrm{GeLU}\left(W_1\LN\left(x_i+\Attn(\LN(x))_i\right)+\lambda_1\right)+\lambda_2.
    \end{align*}
    Hence the only part of $\TF(x)_i$ that can grow directly with $x_i$ is the residual term $x_i$.

    We now bound the two remaining terms. Since LayerNorm has bounded image, there exists $B_{\LN}<\infty$ such that
    \begin{align*}
        \|\LN(z)\|\leq B_{\LN},\quad\text{for all }z\in\R^d.
    \end{align*}
    For the attention branch,
    \begin{align*}
        \Attn(\LN(x))_i
        =
        \frac{1}{Z_i}\sum_{j=1}^i\exp\left(\LN(x_j)^\top K^\top Q\LN(x_i)\right)V\LN(x_j),
    \end{align*}
    where the coefficients $\exp(\cdot)/Z_i$ are nonnegative and sum to one. Thus $\Attn(\LN(x))_i$ is a convex combination of vectors $V\LN(x_j)$, and
    \begin{align*}
        \|\Attn(\LN(x))_i\|
        \leq
        \|V\|B_{\LN}.
    \end{align*}
    In particular, if
    \begin{align*}
        c_i=x_i+\Attn(\LN(x))_i,
    \end{align*}
    then $c_i-x_i$ is uniformly bounded, independently of the sequence length and the input.

    For the MLP branch, the input is $\LN(c_i)$, which again lies in the bounded image of LayerNorm. Therefore $W_1\LN(c_i)+\lambda_1$ lies in a bounded subset of $\R^{d'}$. Since scalar GeLU is bounded on bounded intervals, the map
    \begin{align*}
        u\mapsto W_2\mathrm{GeLU}(W_1u+\lambda_1)+\lambda_2
    \end{align*}
    is bounded on $\LN(\R^d)$. Hence there exists $B_{\mathrm{MLP}}<\infty$ such that
    \begin{align*}
        \left\|W_2\mathrm{GeLU}(W_1\LN(c_i)+\lambda_1)+\lambda_2\right\|
        \leq
        B_{\mathrm{MLP}}.
    \end{align*}

    Combining the attention and MLP bounds, there exists a constant
    \begin{align*}
        C_{\mathrm{block}}=\|V\|B_{\LN}+B_{\mathrm{MLP}}<\infty
    \end{align*}
    such that for every sequence length $m$, every input sequence $x=(x_1,\cdots,x_m)$, and every position $i\in[m]$,
    \begin{align}\label{eq:bounded-transformer-residual}
        \|\TF(x)_i-x_i\|\leq C_{\mathrm{block}}.
    \end{align}

    Since $\mathcal{T}_\theta$ is a finite composition of such blocks, summing the bounds over blocks gives a constant $C_\theta<\infty$ such that, for every sequence length $m$, every sequence $x$, and every position $i$,
    \begin{align}\label{eq:bounded-full-transformer-residual}
        \|\mathcal{T}_\theta(x)_i-x_i\|\leq C_\theta.
    \end{align}

    Now take any prompt $a=(a_1,\cdots,a_n)$ and let $\mathrm{AR}_{\theta,n}(a)=(b_1,b_2)$. In the second autoregressive step, the input at position $n+1$ is exactly $b_1$. Applying \Cref{eq:bounded-full-transformer-residual} to the sequence $(a_1,\cdots,a_n,b_1)$ at position $n+1$ gives
    \begin{align*}
        \|b_2-b_1\|
        =
        \|\mathcal{T}_\theta(a_1,\cdots,a_n,b_1)_{n+1}-b_1\|
        \leq C_\theta.
    \end{align*}
    Thus the image of $\mathrm{AR}_{\theta,n}$ is contained in
    \begin{align*}
        \left\{(y_1,y_2)\in\R^d\oplus\R^d\mid \|y_2-y_1\|\leq C_\theta\right\}.
    \end{align*}
    This is a proper subset of $\R^d\oplus\R^d$: for example, any pair $(0,y)$ with $\|y\|>C_\theta$ is not in the image. Hence $\mathrm{AR}_{\theta,n}$ is not surjective. Since this holds for every choice of finite parameters $\theta$, the class is not almost always surjective.
\end{proof}

\subsection{More Applications of Degree Theory}
We start with an alternative proof of \Cref{thm:preln} using Brouwer degree.
\begin{proof}[Alternative proof of \Cref{thm:preln}]
    We give another proof of \Cref{thm:preln} using Brouwer degree. Let
    \begin{align*}
        g(x)=x+f(\LN(x)),\quad h(x)=f(\LN(x)).
    \end{align*}
    By the boundedness of the image of $\LN$ and the continuity of $f$, the map $h$ is bounded. Write
    \begin{align*}
        M=\sup_{x\in\R^d}\|h(x)\|<\infty.
    \end{align*}
    Fix an arbitrary target $y\in\R^d$, choose $R>\|y\|+M$, and let $\Omega=\{x\in\R^d\mid \|x\|<R\}$. Consider the homotopy
    \begin{align*}
        g_t(x)=x+t f(\LN(x)),\quad t\in[0,1].
    \end{align*}
    For every $x\in\partial\Omega$ and every $t\in[0,1]$,
    \begin{align*}
        \|g_t(x)\|\geq \|x\|-t\|h(x)\|\geq R-M>\|y\|.
    \end{align*}
    Hence $y\notin g_t(\partial\Omega)$ throughout the homotopy. By \Cref{lem:degree},
    \begin{align*}
        \deg(g,\Omega,y)=\deg(g_0,\Omega,y)=\deg(\mathrm{Id},\Omega,y)=1,
    \end{align*}
    where the last equality follows because $y\in\Omega$ and the identity map has one pre-image of $y$ with positive orientation. Since the degree is nonzero, $g^{-1}(y)$ is nonempty. As $y$ was arbitrary, $g$ is surjective.
\end{proof}

In the following example we discuss MLP with LeakyReLU activation, another widely used activation function. Although the homotopy does not yield a clean surjectivity theorem for general MLPs, it reveals a simple sufficient condition that may be useful elsewhere.

\begin{definition}
    For $\alpha\in(0,1]$, the LeakyReLU activation $\mathrm{LeakyReLU}_\alpha:\R^d\to\R^d$ is defined entrywise by
    \begin{align*}
        \mathrm{LeakyReLU}_\alpha(z)_i=\max\{z_i,\alpha z_i\}.
    \end{align*}
\end{definition}

\begin{theorem}\label{thm:leakyrelu-mlp-sufficient}
    Let $\alpha\in(0,1]$ and let
    \begin{align*}
        f(x)=W_2\mathrm{LeakyReLU}_\alpha(W_1x+\lambda_1)+\lambda_2
    \end{align*}
    be a two-layer MLP from $\R^d$ to $\R^d$, where $W_1\in\R^{d_1\times d}$, $W_2\in\R^{d\times d_1}$, $\lambda_1\in\R^{d_1}$, and $\lambda_2\in\R^d$. If
    \begin{align*}
        \det(W_2DW_1)\neq0
    \end{align*}
    for every diagonal matrix $D=\mathrm{diag}(s_1,\cdots,s_{d_1})$ with $s_i\in[\alpha,1]$, then $f$ is surjective.
\end{theorem}
\begin{proof}
    Fix $v\in\R^d$. Let
    \begin{align*}
        f^*(x)=W_2(W_1x+\lambda_1)+\lambda_2 .
    \end{align*}
    Since the assumption includes $D=I^{d_1}$, the matrix $W_2W_1$ is invertible. Hence $f^*(x)=v$ has a unique solution
    \begin{align*}
        x^*=(W_2W_1)^{-1}(v-\lambda_2-W_2\lambda_1),
    \end{align*}
    and for every open ball $\Omega$ containing $x^*$,
    \begin{align*}
        \deg(f^*,\Omega,v)=\mathrm{sgn}\det(W_2W_1)\neq0.
    \end{align*}

    Define $\eta_t=1-t(1-\alpha)$ and
    \begin{align*}
        F(x,t)=W_2\sigma_t(W_1x+\lambda_1)+\lambda_2,\qquad
        \sigma_t(z)_i=\max\{z_i,\eta_t z_i\}.
    \end{align*}
    Then $F(x,0)=f^*(x)$ and $F(x,1)=f(x)$. By compactness of the diagonal matrices with entries in $[\alpha,1]$,
    \begin{align*}
        m=\min_D\sigma_{\min}(W_2DW_1)>0.
    \end{align*}
    For each $x,t$, there is such a diagonal matrix $D_{x,t}$ satisfying
    \begin{align*}
        \sigma_t(W_1x+\lambda_1)=D_{x,t}(W_1x+\lambda_1).
    \end{align*}
    Therefore
    \begin{align*}
        \|F(x,t)\|
        \geq m\|x\|-\|W_2\|\|\lambda_1\|-\|\lambda_2\|.
    \end{align*}
    Choose
    \begin{align*}
        R>\max\left\{\|x^*\|,\frac{\|v\|+\|W_2\|\|\lambda_1\|+\|\lambda_2\|}{m}\right\}
    \end{align*}
    and set $\Omega=\{x\in\R^d\mid \|x\|<R\}$. For all $x\in\partial\Omega$ and $t\in[0,1]$, we have $\|F(x,t)\|>\|v\|$, so $v\notin F(\partial\Omega,t)$. By homotopy invariance,
    \begin{align*}
        \deg(f,\Omega,v)=\deg(f^*,\Omega,v)\neq0.
    \end{align*}
    Thus $f(x)=v$ has a solution. Since $v$ was arbitrary, $f$ is surjective.
\end{proof}

\section{Experiment}
Our results prove surjectivity of practical architectures by proving surjectivity of their building blocks. Hence our proof also provides algorithms to find an input corresponding to a specific output of the network that we want. For GPT-2, or more generally surjective autoregressive models, we can find the input one token at a time as described in \Cref{alg}.

\begin{algorithm}
\caption{Finding Input Sequence}
\begin{algorithmic}\label{alg}
\REQUIRE A Frozen Transformer $\TF$, An Output Sequence $b$
\ENSURE A Reconstructed Sequence $a$
\STATE $a_1\gets 0$
\STATE Optimize $a_1$ using gradient descent on loss $\Sp{b_1-\TF(a_1)}^2$
\FOR{$i = 2$ to $n$}
    \STATE $a_i\gets 0$
    \STATE Optimize $a_i$ using gradient descent on loss $\Sp{b_i-\TF(a_1,\cdots,a_i)_i}^2$
\ENDFOR
\RETURN $a$
\end{algorithmic}
\end{algorithm}

In this section, we conduct experiments to verify the theoretical statements using GPT-2. In particular we implement~\Cref{alg} to find inputs corresponding to the following outputs:
\begin{itemize}
    \item Twenty sentences from 2025 New York Times articles. GPT-2 could not have been trained on such sentences. One example is ``The United States and China said Monday they reached an agreement ... threatening the world's two largest economies''. The length of these sentences varies from 7 words to 37 words.
    \item Twenty-five sentences that contain completely random words from vocabulary, with length 2, 4, ..., 50. One example is "whims produ ether debunked depressive FoundingeeshedonApplication Weight refin 58".
\end{itemize}

Notice that~\Cref{alg} cannot guarantee that we always find a corresponding $a$ because gradient-based optimization is still a heuristic algorithm. However, decoding one input at a time is a lot simpler than jointly optimizing the whole sequence $a$ at once. If this optimization is still too difficult, we can further decompose the algorithm into finding the hidden embeddings iteratively. However we find that for GPT-2 \Cref{alg} is enough.

For every gradient descent, we set learning rate to be 0.1 and optimize for 200 steps. We use an A100 GPU for inference. For all forty-five sentences we described above, the algorithm succeeds in finding the corresponding input sequences. The decoding speed per token is $10.25\pm 0.14$ seconds.

\end{document}